\RequirePackage{fix-cm}
\documentclass[12pt]{article}
\PassOptionsToPackage{table,dvipsnames}{xcolor}

\usepackage[left=2.5cm,
right=2.5cm,
top=2.3cm,
bottom=2.3cm,
headheight=20pt,
headsep=10pt,
footskip=25pt,
letterpaper]{geometry}
\usepackage[utf8]{inputenc}
\usepackage[T1]{fontenc}
\usepackage[english]{babel}
\usepackage{amsmath,amsfonts,amssymb,amsthm,thmtools}
\usepackage{graphicx}
\usepackage{hyperref}
\usepackage{fancyhdr}
\usepackage[normalem]{ulem}
\usepackage{graphicx}
\usepackage{stfloats}
\usepackage{wrapfig}
\usepackage{epstopdf}
\usepackage{cleveref}
\usepackage{subfloat}
\usepackage{subcaption}
\usepackage{xspace}
\usepackage{enumitem}
\usepackage{listings}
\usepackage{titlesec}
\usepackage{etoolbox}
\usepackage{setspace}
\usepackage{changepage}
\usepackage{etoolbox}
\usepackage{multirow}
\usepackage{booktabs}
\usepackage{tabularx}
\usepackage{wrapfig}
\usepackage{svg}
\usepackage[percent]{overpic}
\usepackage[round]{natbib}
\usepackage[colorinlistoftodos, shadow,color=blue!30!white
]{todonotes}
\usepackage{xpatch}
\usepackage{siunitx}

\renewcommand{\sfdefault}{phv}

\fancypagestyle{first}{\fancyfoot[R]{\small\thepage}}

\setlist[itemize]{leftmargin=1em,itemsep=0ex,topsep=0ex}
\titlespacing*{\paragraph}{0pt}{0ex plus .1ex}{1ex}
\titlespacing*{\section}{0ex}{2.3ex plus .3ex minus .0ex}{.6ex plus .3ex minus .2ex}
\titlespacing*{\subsection}{0ex}{1.5ex plus .3ex minus .5ex}{.4ex plus .2ex minus .1ex}
\titlespacing*{\subsubsection}{0ex}{1.2ex plus .3ex minus .3ex}{.3ex plus .2ex minus .2ex}

\xapptocmd\normalsize{%
\abovedisplayskip=.8em plus .2em minus .2em
\belowdisplayskip=.6em plus .1em minus .1em
\abovedisplayshortskip=.8em plus .2em minus .2em
\belowdisplayshortskip=.6em plus .1em minus .1em
}{}{}

\setcitestyle{numbers}
\renewcommand{\cite}[1]{\citep{#1}}

\definecolor{mydarkblue}{rgb}{0.0,0.15,0.7}
\hypersetup{%
colorlinks=true,
linkcolor=mydarkblue,
citecolor=mydarkblue,
filecolor=mydarkblue,
urlcolor=mydarkblue}

\makeatletter

  \renewcommand{\maketitle}{%
    \begingroup
      {\centering\LARGE\@title\par}%
      \vskip 1em
      \centering
      \begin{tabular}[t]{@{}c@{}}\strut\@author\strut\end{tabular}%
      \vskip 0.3in minus 0.1in
    \endgroup
  }
\makeatother

\usepackage{amsmath,amsfonts,bm}

\def\eqref#1{equation~\ref{#1}}

\def\1{\bm{1}}

\DeclareMathAlphabet{\mathsfit}{\encodingdefault}{\sfdefault}{m}{sl}
\SetMathAlphabet{\mathsfit}{bold}{\encodingdefault}{\sfdefault}{bx}{n}

\renewcommand{\eqref}[1]{\textup{(\ref{#1})}}

\usepackage{amsmath,amssymb}
\usepackage{graphicx}
\usepackage{wrapfig}
\usepackage{booktabs}
\usepackage{multirow}
\usepackage{tabularx}
\usepackage[table]{xcolor}
\usepackage{hyperref}
\usepackage{url}

\usepackage{amsthm}
\usepackage{longtable}
\usepackage{placeins}
\newtheorem{applemma}{Lemma}[section]
\newtheorem{apptheorem}[applemma]{Theorem}
\newtheorem{appproposition}[applemma]{Proposition}
\newtheorem{appcorollary}[applemma]{Corollary}

\newcommand{\method}{\textsc{ICLR}}

\definecolor{oursrow}{RGB}{245,245,252}

\definecolor{groupgray}{RGB}{243,243,243}
\definecolor{gainred}{RGB}{190,45,45}
\definecolor{gaingreen}{RGB}{35,150,85}
\definecolor{algblue}{RGB}{25,55,220}
\definecolor{algred}{RGB}{225,35,25}
\newcommand{\goodgain}[1]{\hspace{2pt}{\scriptsize\textcolor{gaingreen}{(#1)}}}
\newcommand{\badgain}[1]{\hspace{2pt}{\scriptsize\textcolor{gainred}{(#1)}}}
\newcommand{\neutralgain}[1]{\hspace{2pt}{\scriptsize\textcolor{gray}{(#1)}}}

\newcolumntype{Y}{>{\centering\arraybackslash}X}
\title{When Can Agents Forget Their Reasoning? ICLR for Long-Horizon Agent Context Compression}
\hypersetup{pdftitle={When Can Agents Forget Their Reasoning? ICLR for Long-Horizon Agent Context Compression}}

\author{%
  \mbox{Mingxuan Wang\textsuperscript{1}}\quad \mbox{Fei Luo\textsuperscript{1}}\quad \mbox{Bo Wang\textsuperscript{1}}\quad \mbox{Guorun Yao\textsuperscript{1}}\quad \mbox{Yinglong Guo\textsuperscript{1}}\\[2pt]
  \mbox{Chao Ning\textsuperscript{1}}\quad \mbox{Hongyue Chen\textsuperscript{1}}\quad \mbox{Yanbiao Ma\textsuperscript{2,*}}\quad \mbox{Jungong Han\textsuperscript{3,*}}\\[4pt]
  \textsuperscript{1}TierFlow Team\\
  \textsuperscript{2}Gaoling School of Artificial Intelligence, Renmin University of China\\
  \textsuperscript{3}Tsinghua University\\[3pt]
  \textsuperscript{*}Corresponding authors.\quad \href{mailto:ybma1998@ruc.edu.cn}{\texttt{ybma1998@ruc.edu.cn}}%
}
\hypersetup{pdfauthor={Mingxuan Wang, Fei Luo, Bo Wang, Guorun Yao, Yinglong Guo, Chao Ning, Hongyue Chen, Yanbiao Ma, Jungong Han}}

\date{}

\usepackage{team-template}
\renewcommand{\TeamPaperID}{ICLR / LONG-HORIZON AGENTS}
\renewcommand{\TeamShortTitle}{When Can Agents Forget Their Reasoning?}
\begin{document}
\pagestyle{fancy}

\maketitle
\thispagestyle{first}
\suppressfloats[t]

\begin{abstract}
Long horizon language model agents continually accumulate reasoning history, increasing context length and inference cost even after earlier decisions have been executed and observed. Unlike static Chain of Thought compression, removing historical reasoning can change future actions and the resulting interaction trajectory. We study when such reasoning can be safely forgotten. We propose \textbf{Interaction Aware Compression for Long Horizon Reasoning (ICLR)}, a training free online method that ranks reasoning blocks using frozen proxy entropy while preserving actions, tool calls, and observations. On 260 WorkBuddyBench tasks, ICLR improves average reward from $0.699$ to $0.718$, while reducing input, output, and cache read tokens by $25.5\%$, $14.4\%$, and $33.3\%$, respectively. Ablations reveal trajectory amplification, where local reasoning deletion produces nonlinear changes in total computation by altering subsequent interaction. Representation probing, activation patching, and controlled trajectory analyses further suggest that historical reasoning becomes more replaceable once task relevant derived state has been reliably externalized into code, files, tool outputs, or environmental feedback. These results characterize agent reasoning as dynamic working state rather than permanent interaction history.
\end{abstract}

\begin{figure*}[t]
    \centering
    \includegraphics[width=\textwidth]{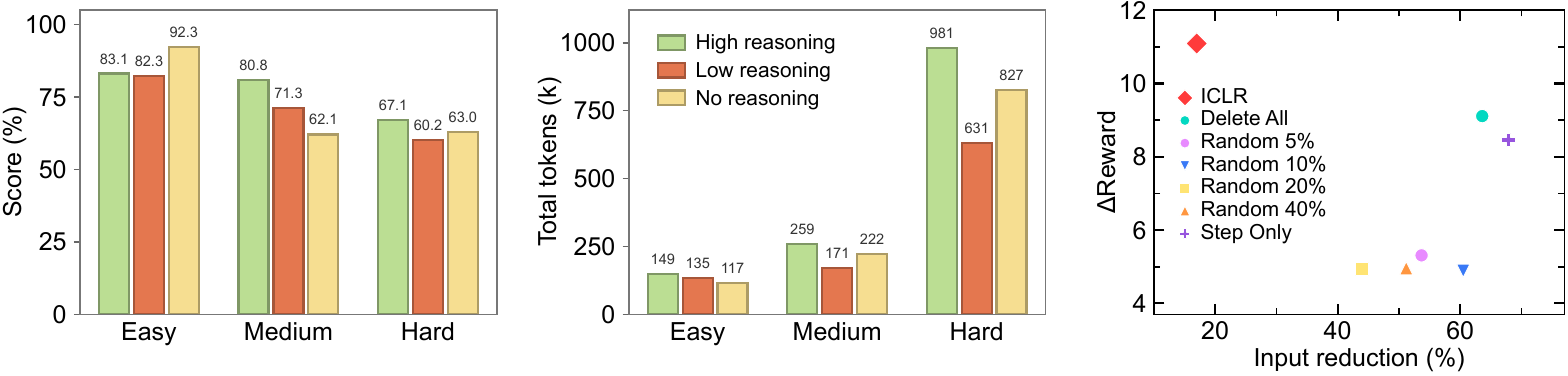}
    
    \caption{
    \textbf{Reasoning utility and cost are state dependent.}
    \textbf{(a,b)} Scores and token usage vary across reasoning settings and task difficulty.
    \textbf{(c)} Online interventions show a nonmonotonic reward--efficiency tradeoff.
    See Appendices~\ref{app:motivation} and~\ref{app:ablations}.
    }
    \label{fig:motivation}
    
\end{figure*}

\section{Introduction}
\label{sec:introduction}

Large language models are increasingly used in complex agent tasks, where reasoning, action, and environmental feedback form a repeated interaction loop~\citep{yao2022react,shinn2023reflexion,park2023generative}. As a task progresses, earlier reasoning is repeatedly included in later requests, increasing context length and inference cost~\citep{kang2025acon,wei2026pace,wu2026contextbudget}. Meanwhile, new observations continuously change the task state. This raises a central question for agent context management: once reasoning has produced an action and the environment has returned feedback, does that reasoning still need to remain in the active context?

Figure~\ref{fig:motivation} motivates this question from two observations. The most effective reasoning setting varies with task difficulty, and more reasoning does not consistently yield higher scores. Our online interventions also show a nonmonotonic relation among reasoning removal, reward, and input reduction. Together, these observations suggest that the value of reasoning depends on the current state rather than on reasoning quantity alone.

Prior work has shown that explicit Chain of Thought reasoning contains substantial redundancy~\citep{wei2022chain,wang2022self,li2026making}, and existing methods can shorten intermediate reasoning while preserving answer quality~\citep{li2026making,li2026chain,qiao2025concise,xia2025tokenskip}. These settings are largely static because a completed reasoning trace can be modified and then evaluated against the same final answer. Agent reasoning is different. Changing reasoning at step $t$ may alter the next tool call, observation, and subsequent decisions. Reasoning compression for agents therefore intervenes in a closed loop decision process rather than simply modifying text~\citep{li2026self,wang2026swe,li2026sculptor}.

Reasoning can move across different forms of state. A plan, constraint, or intermediate conclusion may exist only in the reasoning trace, but become available through tool calls, code, files, or environmental feedback. Context management and memory systems use persistent or external state to reduce dependence on the active context~\citep{packer2023memgpt,hu2026sam,li2026acm,lin2026context}. We therefore hypothesize that the future necessity of reasoning depends on the current agent state and on whether task relevant information has already been externalized.

We propose \textbf{Interaction Aware Compression for Long Horizon Reasoning (ICLR)}, a training free compression method for long horizon agents. After each interaction step, a frozen proxy model estimates the predictive entropy of reasoning blocks and preferentially compresses low entropy blocks~\citep{li2026making}. The compressed history is directly used by the next agent request. Across long horizon agent tasks from multiple domains, ICLR improves reward from $0.699$ to $0.718$, while reducing input, output, and cache read tokens by $25.5\%$, $14.4\%$, and $33.3\%$.

Our ablations reveal that local deletion and total system computation are not proportional. Small changes to reasoning history can produce much larger changes in token consumption, while removing all reasoning does not necessarily produce the shortest trajectory because local interventions alter later decisions. We refer to this effect as \emph{trajectory amplification}.

To understand which reasoning should remain available, we further study its \emph{future necessity}. Using representation probing techniques~\citep{alain2016understanding,hewitt2019designing,belinkov2022probing}, we find that the hidden state at the current context boundary contains decodable information associated with later rederivation, additional reasoning, and replanning, reaching an AUROC of $0.844$ under strict task level nested evaluation. Activation patching provides complementary evidence~\citep{meng2022locating}: replacing deleted history representations with their full history counterparts progressively restores the perturbed next token distribution, with KL Repair increasing from $0.047$ at layer 4 to $0.459$ at layer 20, $0.724$ at layer 28, and $0.978$ at layer 31.

Trajectory analysis further suggests an \emph{internalization and externalization gap}. When task specific derived state remains available only in reasoning, future behavioral risk is $67.3\%$, compared with $36.2\%$ after that state has been externalized. Controlled interventions support the same interpretation: keeping only the latest reasoning preserves the full history rule score, deleting it reduces the score, clearing reasoning after relevant code has been written leaves the final score unchanged, and hiding observed environmental feedback causes the agent to recover the same information through an additional interaction.

Together, these findings suggest that historical reasoning is most valuable when it remains the only reliable carrier of task relevant derived state. As that state is externalized, the original reasoning becomes increasingly replaceable.

Our contributions are threefold.

\textbf{(1) Training free online reasoning compression for agents.}
We introduce ICLR, which continuously compresses accumulated reasoning inside a real agent interaction loop while preserving actions, tool calls, and observations.

\textbf{(2) Trajectory level characterization of agent reasoning compression.}
Through full benchmark evaluation and systematic ablations, we show that reasoning compression can alter the execution trajectory, producing a nonlinear relation between local deletion and total system computation.

\textbf{(3) Analysis of future reasoning necessity.}
We study historical reasoning through hidden state probing, activation patching, trajectory analysis, and information externalization. The results support a view of agent reasoning as dynamic working state whose future value depends on where task relevant information is stored.

\section{Related Work}
\label{sec:related_work}

\subsection{Efficient Reasoning and Chain of Thought Compression}

Extended Chain of Thought reasoning has become a central mechanism for mathematical reasoning, code generation, and complex decision making, but longer reasoning traces also increase token usage, latency, and inference cost~\citep{wei2022chain,wang2022self,hao2024training,zelikman2024quiet}. A broad line of work therefore studies how much explicit reasoning is actually necessary, including adaptive reasoning budgets, early stopping, concise reasoning, and direct pruning of intermediate reasoning steps~\citep{li2026making,li2026chain,qiao2025concise,sui2026think,xia2025tokenskip}. These studies collectively suggest that reasoning length and task performance are not simply proportional. A complete reasoning trace often mixes critical deductions with repeated checks, confirmations, and procedural elaboration. Identifying which parts can be removed without harming downstream behavior has therefore become an important direction in efficient reasoning.

Recent work further estimates redundancy at the level of individual reasoning steps. Step Entropy~\citep{li2026making} ranks reasoning steps according to uncertainty in the predictive distribution and shows that many low entropy steps can be removed from mathematical reasoning traces while preserving final answer accuracy. CONCISE~\citep{qiao2025concise}, TokenSkip~\citep{xia2025tokenskip}, and related compression methods similarly reduce explicit reasoning by identifying steps or tokens that contribute less to the final solution. These methods mainly optimize reasoning within a single completed generation. The decision to remove a step is therefore evaluated against a fixed downstream answer rather than an evolving sequence of actions and observations.

Our setting differs because historical reasoning remains part of the agent's future decision state. Removing a reasoning block can change the next action, which changes the next observation and all later decisions. The relevant objective is therefore not only whether compressed reasoning still supports the same answer, but whether repeated compression remains effective throughout a closed loop interaction.

\subsection{Context Management and Memory for Language Model Agents}

The growth of interaction history has motivated context management methods for long horizon agents~\citep{liu2024lost,chen2024longlora,munkhdalai2024leave,xiao2024efficient,zhang2023h2o}. Token level prompt compression methods such as LLMLingua~\citep{jiang2023llmlingua}, LongLLMLingua~\citep{jiang2024longllmlingua}, LLMLingua-2~\citep{pan2024llmlingua}, context compression~\citep{li2023compressing}, and gist tokens~\citep{mu2023learning} reduce redundant context. Agent oriented approaches manage evolving histories through relevance estimation, summarization, adaptive pruning, explicit context operations, or external memory~\citep{sun2025scaling,wu2025resum,ye2025agentfold,verma2026active}. Representative examples include ACON~\citep{kang2025acon}, PACE~\citep{wei2026pace}, SWE Pruner~\citep{wang2026swe}, Self Compact~\citep{li2026self}, Self GC~\citep{hao2026self}, Sculptor~\citep{li2026sculptor}, ContextBudget~\citep{wu2026contextbudget}, Context as a Tool~\citep{liu2026context}, ARC~\citep{yao2026arc}, and structured context eviction~\citep{semenov2026beyond}.

A complementary line of work uses persistent or external memory to reduce dependence on the active context. MemGPT~\citep{packer2023memgpt}, SAM~\citep{hu2026sam}, CoMem~\citep{zhang2026comem}, ACM~\citep{li2026acm}, Mem1~\citep{zhou2026mem1}, proactive memory extraction~\citep{yang2026beyond}, and transferable agent memory~\citep{liang2026learning} preserve or recover information outside the immediate interaction window. These systems ask what information should remain accessible. We instead ask whether reasoning that has already produced an action still needs to remain after its consequences have been observed.

This distinction follows from the information flow of an agent. Reasoning combines available evidence into plans and intermediate conclusions, while actions and observations can move the same task relevant information into code, files, tool outputs, or environmental feedback~\citep{lin2026context,liu2026context,yao2026arc}. Historical reasoning may therefore change from being the only carrier of useful derived state to being redundant with information already stored elsewhere. Prior work on hidden state probing~\citep{alain2016understanding,hewitt2019designing,belinkov2022probing,burns2022discovering} and activation interventions~\citep{meng2022locating} provides tools for studying this transition. We use them to examine whether the current model state predicts future reasoning reuse and whether reasoning becomes more replaceable after task relevant state is externalized.

\section{Methodology}
\label{sec:method}

\begin{figure*}[t]
    \centering
    \includegraphics[width=0.98\textwidth]{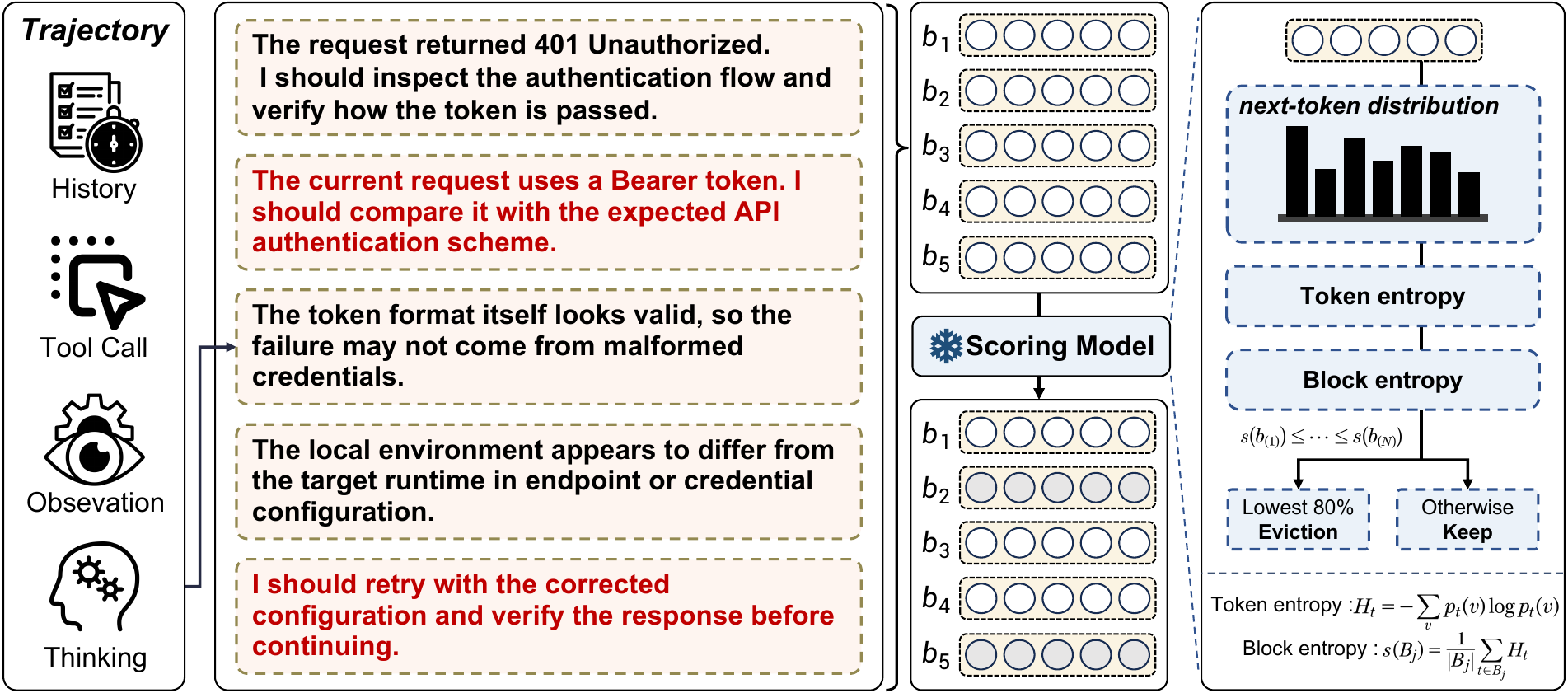}
    
    \caption{
    \textbf{Overview of ICLR.}
    At each completed agent step, newly generated reasoning is partitioned into blocks while actions, tool calls, and observations are preserved.
    A frozen proxy scorer estimates token level predictive entropy for each block.
    Low entropy blocks are removed before the compressed history is written back for the next agent request.
    }
    \label{fig:overview}
    
\end{figure*}

Figure~\ref{fig:overview} summarizes the online compression process.
ICLR operates inside the interaction loop rather than on a completed trajectory.
It reduces accumulated reasoning while preserving the external records that define the evolving task state.
We first describe online compression and entropy based selection, then the analyses used to study when historical reasoning remains useful.

\subsection{Online Reasoning Compression}
\label{sec:online_compression}

We consider a long horizon agent that repeatedly reasons, acts, and receives environmental feedback.
Let the $t$th interaction step be
\begin{equation}
S_t=(R_t,A_t,O_{t+1}),
\end{equation}
where $R_t$ is the generated reasoning, $A_t$ the resulting action or tool call, and $O_{t+1}$ the returned observation.
Before the $k$th model request, the interaction history is
\begin{equation}
\mathcal H_k=(P,S_1,\ldots,S_{k-1}),
\end{equation}
where $P$ contains the system prompt, tool definitions, and task instruction.

Static Chain of Thought compression modifies a completed reasoning trace before evaluating its final answer.
Online compression instead changes the state from which future actions are generated.
For the $k$th request, ICLR induces
\begin{equation}
\mathcal H_k
\rightarrow
\widetilde{\mathcal H}_k
\rightarrow
A_k
\rightarrow
O_{k+1}
\rightarrow
\mathcal H_{k+1}.
\end{equation}
A local compression decision can therefore affect later tool selection, observations, recovery, and termination.

The first agent request is uncompressed.
Before each subsequent request, ICLR compresses only the reasoning from the just completed interaction step while preserving its action, tool call, and observation.
The compressed reasoning is written back into persistent history and is not restored, so all later model calls operate on the modified trajectory.

Local text removal is not equivalent to total computation saved.
A small change to reasoning history can alter later decisions and therefore change both the number and content of future requests.
Local deletion should thus be distinguished from total computation over the resulting trajectory.

\subsection{Entropy Guided and Context Conditioned Pruning}
\label{sec:entropy_pruning}

For each completed reasoning trace, we partition the reasoning into blocks,
\begin{equation}
R_t=(B_1,B_2,\ldots,B_N).
\end{equation}
The current implementation uses double newline boundaries as a deterministic segmentation rule and excludes empty blocks.
Compression is applied only to reasoning.
Actions, tool arguments, tool outputs, observations, and final answers remain unchanged.
This design prevents the compression policy from directly deleting evidence that has already entered the external interaction record.

\paragraph{Proxy entropy.}
We score each reasoning block using a frozen proxy language model~\citep{li2026making}.
Given scorer context $c_t$, let $z(c_t)$ denote the logits predicting
the next token. The predictive distribution is
\begin{equation}
p_\theta(v\mid c_t)
=
\operatorname{softmax}\!\left(z(c_t)\right)_v.
\end{equation}
The token level predictive entropy is
\begin{equation}
h_t
=
-\sum_{v\in\mathcal V}
p_\theta(v\mid c_t)
\log_2 p_\theta(v\mid c_t),
\end{equation}
and the score of block $B_i$ is the mean entropy of its tokens,
\begin{equation}
H(B_i)
=
\frac{1}{|B_i|}
\sum_{t\in B_i} h_t.
\end{equation}

Lower entropy indicates greater predictive certainty under the frozen proxy model.
We use this \emph{proxy entropy} to rank reasoning blocks for compression.
The acting agent is DeepSeek V4 Flash, while the scorer is a frozen Qwen3.5 9B model.

\paragraph{Context conditioned ranking.}
A reasoning block is not scored in isolation.
The scorer receives the system prompt, tool definitions, accumulated interaction history, and the current completed step in their original causal order.
The score can therefore be viewed as a context dependent quantity,
\begin{equation}
H(B_i\mid\mathcal H_t),
\end{equation}
so the same reasoning text may receive a different score under a different trajectory state.
This is important in agent settings because the relevance of earlier reasoning can change after new actions and observations have modified the task state.

The full context scorer has a $42$K token input limit.
If this limit is exceeded, pruning is skipped for that step rather than truncating the history observed by the acting agent.

\paragraph{Block selection.}
Given $N$ reasoning blocks, ICLR removes the lowest entropy fraction,
\begin{equation}
\mathcal D_t
=
\operatorname{BottomK}_{\lfloor \rho N \rfloor}
\left(
\{H(B_i)\}_{i=1}^{N}
\right),
\qquad
\rho=0.8.
\end{equation}
Each block in $\mathcal D_t$ is replaced with \texttt{[SKIP]}, while all retained blocks preserve their original order and text.
When $N=1$, no block is removed.
Because the rule operates on blocks rather than individual tokens, the realized token reduction varies across interaction steps.

We use a fixed compression ratio rather than tuning a task specific threshold.
This choice creates a consistent intervention across trajectories and allows us to study whether agents can repeatedly discard a large fraction of historical reasoning during real interaction.

\paragraph{Intervention variants.}
We construct three additional variants for controlled comparison.
\emph{Step Only Entropy} scores only the reasoning from the current completed step and removes earlier trajectory context from the scorer.
\emph{Delete All Thinking} removes the complete targeted reasoning trace while preserving actions and observations.
\emph{Random $r$} removes a deterministic random subset of reasoning tokens with
\begin{equation}
r\in\{5\%,10\%,20\%,40\%\}.
\end{equation}
All variants follow the same online write back protocol.
The modified history is therefore consumed by every subsequent real agent request.
These interventions separate the effects of context conditioning, selection strategy, and deletion strength.

\subsection{Future Necessity Probing}
\label{sec:future_necessity_method}

The compression policy determines which reasoning is removed, but it does not explain which reasoning will matter again later.
We therefore study whether the current model state contains information associated with future reasoning reuse before that later behavior occurs.
Following standard representation probing methods~\citep{alain2016understanding,hewitt2019designing,belinkov2022probing}, we analyze the model state available at each compression boundary.

We use a frozen Qwen3.5 9B model as a representation sensor.
For sample $i$, layer $l$, and pooling rule $p$, let
$\mathcal Z_i^{(l)}$ denote the hidden states associated with the current context boundary.
We define the pooled representation as
\begin{equation}
z_i^{(l,p)}
=
\operatorname{Pool}_p
\left(
\mathcal Z_i^{(l)}
\right).
\end{equation}
A linear probe then predicts a behavioral target $y_i$,
\begin{equation}
\widehat p_i
=
\sigma
\left(
w^\top z_i^{(l,p)}+b
\right).
\end{equation}

The target records whether the later trajectory contains rederivation, additional reasoning, or explicit replanning.
We use this behavioral target as an operational measure of future reasoning reuse.
The probe characterizes this signal from the current representation.

To prevent task leakage, evaluation uses strict task level nested splits.
Layer selection, pooling choice, feature standardization, and PCA are selected using training tasks only.
The outer evaluation folds contain held out tasks.
This protocol asks whether the current representation contains information that generalizes across tasks rather than information specific to trajectories already seen during probe selection.
Full extraction and evaluation details are reported in Appendix~\ref{app:probing}.

\subsection{Mechanistic Interventions for Reasoning Persistence}
\label{sec:mechanistic_interventions}

Representation probing tests whether future reasoning reuse can be predicted from the current state.
We complement this analysis with interventions that examine how historical reasoning affects model predictions and how its role changes after task relevant information becomes available outside the reasoning trace.

\paragraph{Activation patching.}
We first test whether representations conditioned on historical reasoning participate in the next token prediction of the frozen sensor model~\citep{meng2022locating}.
For sample $i$, let $P_i^{\rm full}$ denote the next token distribution under the full reasoning history and $P_i^{\rm del}$ the corresponding distribution after historical reasoning is removed.
At layer $l$, we replace the prompt boundary residual under the deleted history condition with the residual from the full history condition.
The resulting distribution is denoted by $P_{i,l}^{\rm patch}$.

Let
\begin{equation}
d_i
=
D_{\rm KL}
\left(
P_i^{\rm full}
\Vert
P_i^{\rm del}
\right).
\end{equation}
For samples with $d_i>0$, we measure distributional recovery as
\begin{equation}
\operatorname{Repair}_i(l)
=
1-
\frac{
D_{\rm KL}
\left(
P_i^{\rm full}
\Vert
P_{i,l}^{\rm patch}
\right)
}{
d_i
}.
\end{equation}
Larger values indicate that the patched representation more closely recovers the full history next token distribution.
The intervention is performed only on the frozen sensor model and does not modify the online ICLR policy.

\paragraph{State externalization.}
We next examine whether reasoning becomes more replaceable after task specific derived state has been written into a persistent external carrier.
Let $\mathcal U_i$ denote the unexternalized task specific derived state at compression boundary $i$.
We define
\begin{equation}
G_i
=
\mathbb{I}
\left(
\mathcal U_i\neq\varnothing
\right).
\end{equation}
Here, derived state refers to an intermediate plan, relation, constraint, synthesis, or task specific conclusion rather than a direct copy of the task instruction or a raw observation.
External carriers include code, files, tool outputs, and explicit environmental feedback.

The indicator $G_i$ separates states in which reasoning remains the only available carrier of task relevant derived information from states in which that information has already been materialized elsewhere.
This allows us to study whether future reasoning reuse changes as information moves from internal reasoning into persistent external state.

\paragraph{Controlled state interventions.}
Finally, we construct fixed state continuations that modify the available reasoning or feedback while keeping the surrounding task state unchanged.
The interventions compare retaining or deleting recent reasoning, clearing reasoning after relevant code has been written to disk, and retaining or hiding previously observed environmental feedback.

These interventions test whether the next decision still depends on information available only in the reasoning trace.
Together, probing, activation patching, state externalization analysis, and controlled interventions provide complementary views of reasoning persistence.
The full protocols are reported in Appendices~\ref{app:patching} and~\ref{app:externalization}.

\section{Experiments}
\label{sec:experiments}

We evaluate \method{} on WorkBuddyBench using DeepSeek V4 Flash as the task executing agent and a frozen Qwen3.5 9B model as the proxy entropy scorer. WorkBuddyBench contains long horizon tasks from Code, Office, Security, and Web domains. We report official task reward together with input tokens, output tokens, cache read tokens, and wins, ties, and losses. The full benchmark contains 260 tasks, consisting of 80 Code, 50 Office, 60 Security, and 70 Web tasks. A fixed 80 task set with 20 tasks from each domain is used for ablation analysis.

\subsection{Overall Performance}

We first compare \method{} with representative context management baselines on the fixed Pilot40 task set used by prior WorkBuddyBench context management experiments. Table~\ref{tab:pilot40} includes simple sliding window baselines, periodic summarization, and representative context management methods including PACE, SelfCompact, ACON Core, SAM, and SWE Pruner. This comparison focuses on the tradeoff between task quality and token consumption. \method{} directly prunes low entropy reasoning from the online interaction history using a frozen scorer, providing a lightweight alternative to summary based or external memory based context management.

\begin{table*}[htbp]
\centering
\caption{
\textbf{Comparison with context management methods on WorkBuddyBench Pilot40.}
All methods use the same fixed 40 tasks. Scores are multiplied by $100$.
Parenthesized values show changes from the DeepSeek-V4-Flash base agent.
}
\label{tab:pilot40}

\small
\setlength{\tabcolsep}{2.4pt}
\renewcommand{\arraystretch}{1.10}

\begin{tabular*}{\textwidth}{@{\extracolsep{\fill}}l|ccccc|c@{}}
\toprule
Method & Code & Office & Sec. & Web & Avg. & Tokens $\downarrow$ \\
\midrule

\rowcolor{groupgray}
\multicolumn{7}{c}{\textbf{Base and Ours}} \\
\midrule

DeepSeek-V4-Flash
& 72.9
& 81.6
& 44.6
& 71.0
& 67.5
& 1.89M
\\

\rowcolor{oursrow}
\textbf{\method{} (Ours)}
& 74.1 \goodgain{+1.2}
& 81.1 \badgain{-0.5}
& \textbf{64.1} \goodgain{+19.5}
& 62.0 \badgain{-9.0}
& \textbf{70.3} \goodgain{+2.8}
& 2.02M \badgain{+6.9\%}
\\

\midrule
\rowcolor{groupgray}
\multicolumn{7}{c}{\textbf{Simple Context Baselines}} \\
\midrule

Sliding Window ($K=5$)
& 66.8 \badgain{-6.1}
& 67.4 \badgain{-14.2}
& 32.3 \badgain{-12.3}
& 59.0 \badgain{-12.0}
& 56.4 \badgain{-11.2}
& 2.95M \badgain{+56.3\%}
\\

Sliding Window ($K=10$)
& 74.4 \goodgain{+1.4}
& 73.1 \badgain{-8.6}
& 40.1 \badgain{-4.5}
& 61.0 \badgain{-10.0}
& 62.1 \badgain{-5.4}
& 1.98M \badgain{+5.1\%}
\\

Sliding Window ($K=20$)
& 79.3 \goodgain{+6.4}
& 85.1 \goodgain{+3.5}
& 40.3 \badgain{-4.3}
& 70.0 \badgain{-1.0}
& 68.7 \goodgain{+1.1}
& 2.02M \badgain{+7.0\%}
\\

Periodic Summary ($n=3$)
& 65.4 \badgain{-7.5}
& 67.7 \badgain{-14.0}
& 26.3 \badgain{-18.3}
& \textbf{82.7} \goodgain{+11.7}
& 60.5 \badgain{-7.0}
& 2.21M \badgain{+17.2\%}
\\

Periodic Summary ($n=5$)
& \textbf{83.6} \goodgain{+10.7}
& 75.7 \badgain{-5.9}
& 40.0 \badgain{-4.5}
& 81.3 \goodgain{+10.3}
& \textbf{70.2} \goodgain{+2.6}
& 2.37M \badgain{+25.5\%}
\\

\midrule
\rowcolor{groupgray}
\multicolumn{7}{c}{\textbf{API-based Methods}} \\
\midrule

PACE~\citep{wei2026pace}
& 70.4 \badgain{-2.6}
& 55.8 \badgain{-25.9}
& 43.5 \badgain{-1.0}
& 64.0 \badgain{-7.0}
& 58.4 \badgain{-9.1}
& 1.10M \goodgain{-41.6\%}
\\

LLMLingua-2~\citep{pan2024llmlingua}
& 64.6 \badgain{-8.3}
& 73.3 \badgain{-8.3}
& 40.1 \badgain{-4.5}
& 70.0 \badgain{-1.0}
& 62.0 \badgain{-5.5}
& 2.40M \badgain{+27.3\%}
\\

SelfCompact~\citep{li2026self}
& 76.8 \goodgain{+3.9}
& 79.6 \badgain{-2.0}
& 36.2 \badgain{-8.4}
& 71.0 \neutralgain{0.0}
& 65.9 \badgain{-1.6}
& 1.57M \goodgain{-17.0\%}
\\

ACON-Core~\citep{kang2025acon}
& 71.8 \badgain{-1.1}
& 68.6 \badgain{-13.1}
& \textbf{47.9} \goodgain{+3.4}
& 65.0 \badgain{-6.0}
& 63.3 \badgain{-4.2}
& 1.40M \goodgain{-25.7\%}
\\

Self-GC~\citep{hao2026self}
& 63.5 \badgain{-9.5}
& 71.7 \badgain{-10.0}
& 43.7 \badgain{-0.9}
& 71.0 \neutralgain{0.0}
& 62.5 \badgain{-5.1}
& 1.66M \goodgain{-12.4\%}
\\

LRE~\citep{jahan2026learning}
& 62.6 \badgain{-10.3}
& 63.4 \badgain{-18.2}
& 27.1 \badgain{-17.4}
& 68.0 \badgain{-3.0}
& 55.3 \badgain{-12.2}
& 2.29M \badgain{+21.3\%}
\\

CoMem~\citep{zhang2026comem}
& 56.4 \badgain{-16.5}
& 59.9 \badgain{-21.8}
& 2.5 \badgain{-42.1}
& 58.0 \badgain{-13.0}
& 44.2 \badgain{-23.3}
& \textbf{0.77M} \goodgain{-59.5\%}
\\

SAM~\citep{hu2026sam}
& 67.6 \badgain{-5.3}
& 81.6 \neutralgain{0.0}
& 47.2 \goodgain{+2.7}
& 67.0 \badgain{-4.0}
& 65.9 \badgain{-1.7}
& 1.35M \goodgain{-28.7\%}
\\

SWE-Pruner~\citep{wang2026swe}
& 63.5 \badgain{-9.5}
& 72.6 \badgain{-9.0}
& 40.9 \badgain{-3.7}
& 65.0 \badgain{-6.0}
& 60.5 \badgain{-7.1}
& 1.50M \goodgain{-20.5\%}
\\

Sculptor~\citep{li2026sculptor}
& 51.2 \badgain{-21.7}
& 81.6 \neutralgain{0.0}
& 35.2 \badgain{-9.3}
& 69.0 \badgain{-2.0}
& 59.3 \badgain{-8.3}
& 1.29M \goodgain{-31.5\%}
\\

\midrule
\rowcolor{groupgray}
\multicolumn{7}{c}{\textbf{Released-policy Method}} \\
\midrule

ACM~\citep{li2026acm}
& 38.7 \badgain{-34.2}
& 55.9 \badgain{-25.7}
& 6.2 \badgain{-38.4}
& 34.0 \badgain{-37.0}
& 33.7 \badgain{-33.8}
& 2.36M
\\

\bottomrule
\end{tabular*}

\end{table*}

We next evaluate \method{} on the complete 260 task benchmark. Table~\ref{tab:full260} follows the score and token presentation used by prior WorkBuddyBench context management evaluations. Average reward increases from $0.699$ to $0.718$. Detailed token accounting is reported separately in the appendix. The method obtains $102/76/82$ wins, ties, and losses. These results show that aggressive online pruning of reasoning can reduce repeated processing of historical context without sacrificing aggregate task performance across domains.

Domain-level results and component-wise token changes are reported in Appendix~\ref{app:system_results}.

\begin{table*}[htbp]

\centering
\caption{
\textbf{Full benchmark system performance and token usage.}
Scores are multiplied by $100$.
Green and red indicate favorable and unfavorable changes from the base system.
}
\label{tab:full260}

\small
\setlength{\tabcolsep}{4pt}
\renewcommand{\arraystretch}{1.10}

\begin{tabular*}{\textwidth}{@{\extracolsep{\fill}}l|ccccc|c@{}}
\toprule
Model & Code & Office & Sec. & Web & Avg. & Tokens $\downarrow$ \\
\midrule

DeepSeek-V4-Flash
& \textbf{77.0}
& \textbf{81.8}
& 47.8
& \textbf{72.1}
& 69.9
& 2.69M
\\

\rowcolor{oursrow}
\textbf{\method{} (Ours)}
& 70.6 \badgain{-6.4}
& 77.6 \badgain{-4.2}
& \textbf{70.7} \goodgain{+22.9}
& 69.9 \badgain{-2.2}
& \textbf{71.8} \goodgain{+1.9}
& \textbf{2.19M} \goodgain{-18.5\%}
\\

\bottomrule
\end{tabular*}

\end{table*}

\subsection{Ablations and Trajectory Amplification}
\begin{wrapfigure}{r}{0.40\columnwidth}

\centering
\includegraphics[width=\linewidth]{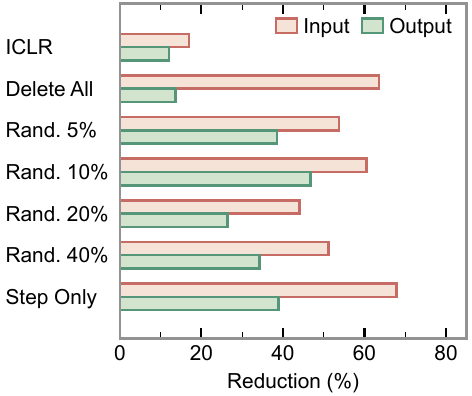}

\caption{\footnotesize\textbf{Trajectory amplification.}
Token savings do not track deletion.}
\label{fig:trajectory}

\end{wrapfigure}
We next study whether the observed effect depends on entropy based selection, access to full history during scoring, or simply reducing the amount of reasoning. Table~\ref{tab:ablation} reports results on the fixed 80 task ablation set. \method{} reaches a reward of $0.738$, compared with $0.627$ for the baseline. Step only entropy reaches $0.711$, while deleting all thinking reaches $0.718$. Random deletion also changes task performance, but none of the reported random variants reaches the reward of full history entropy. These results suggest that reasoning history contains substantial redundancy while also indicating that context conditioned ranking provides useful information beyond deletion alone. Appendix~\ref{app:ablations} reports the full intervention and token summaries used for this analysis.

\begin{table*}[htbp]
\centering
\caption{
\textbf{Ablation results on the fixed 80 task WorkBuddyBench set.}
All variants use the same fixed 80 tasks, with 20 Code, 20 Office, 20 Security, and 20 Web tasks. Scores are multiplied by $100$.
Parenthesized values show changes from the Base Agent.
}
\label{tab:ablation}

\small
\setlength{\tabcolsep}{2.4pt}
\renewcommand{\arraystretch}{1.10}

\begin{tabular*}{\textwidth}{@{\extracolsep{\fill}}l|ccccc|c@{}}
\toprule
Method & Code & Office & Sec. & Web & Avg. & Tokens $\downarrow$ \\
\midrule

\rowcolor{groupgray}
\multicolumn{7}{c}{\textbf{Base and Ours}} \\
\midrule

Base Agent
& 69.09
& 78.94
& 44.17
& 58.50
& 62.67
& 3.00M
\\

\rowcolor{oursrow}
\textbf{\method{} (Ours)}
& 63.85 \badgain{-5.24}
& 81.14 \goodgain{+2.20}
& 72.03 \goodgain{+27.86}
& \textbf{78.00} \goodgain{+19.50}
& \textbf{73.76} \goodgain{+11.09}
& 2.49M \goodgain{-17.0\%}
\\

\midrule
\rowcolor{groupgray}
\multicolumn{7}{c}{\textbf{Reasoning Removal Controls}} \\
\midrule

Delete All Thinking
& \textbf{69.20} \goodgain{+0.11}
& 82.04 \goodgain{+3.10}
& 67.41 \goodgain{+23.24}
& 68.50 \goodgain{+10.00}
& 71.78 \goodgain{+9.11}
& 1.11M \goodgain{-62.9\%}
\\

Random 5\%
& 59.81 \badgain{-9.28}
& 81.12 \goodgain{+2.18}
& 64.17 \goodgain{+20.00}
& 66.83 \goodgain{+8.33}
& 67.98 \goodgain{+5.31}
& 1.40M \goodgain{-53.5\%}
\\

Random 10\%
& 55.08 \badgain{-14.01}
& \textbf{82.65} \goodgain{+3.71}
& 71.91 \goodgain{+27.74}
& 60.61 \goodgain{+2.11}
& 67.56 \goodgain{+4.89}
& 1.19M \goodgain{-60.3\%}
\\

Random 20\%
& 67.23 \badgain{-1.86}
& 76.92 \badgain{-2.02}
& 64.82 \goodgain{+20.65}
& 61.50 \goodgain{+3.00}
& 67.62 \goodgain{+4.95}
& 1.69M \goodgain{-43.8\%}
\\

Random 40\%
& 61.35 \badgain{-7.74}
& 76.17 \badgain{-2.77}
& 61.37 \goodgain{+17.20}
& 71.62 \goodgain{+13.12}
& 67.63 \goodgain{+4.96}
& 1.47M \goodgain{-51.0\%}
\\

\midrule
\rowcolor{groupgray}
\multicolumn{7}{c}{\textbf{Scoring Context Ablation}} \\
\midrule

Step Only Entropy
& 60.48 \badgain{-8.61}
& 80.58 \goodgain{+1.64}
& \textbf{73.10} \goodgain{+28.93}
& 70.38 \goodgain{+11.88}
& 71.14 \goodgain{+8.47}
& 0.98M \goodgain{-67.5\%}
\\

\bottomrule
\end{tabular*}

\end{table*}

The most striking observation is that the local deletion ratio is not monotonically related to final system level computation. Random deletion of $10\%$ of reasoning tokens reduces total input by $60.5\%$, whereas random deletion of $20\%$ reduces input by only $44.0\%$. Deleting all thinking reduces input by $63.6\%$, which is not qualitatively proportional to deleting all reasoning text. Different interventions also produce substantially different numbers of agent calls. We refer to this behavior as \emph{trajectory amplification}: a local change to reasoning history can be amplified through tool selection, observations, repeated reasoning, recovery, and termination, producing a much larger change in the final trajectory. Detailed intervention and token results are reported in Appendix~\ref{app:ablations}.

\subsection{Future Necessity Analysis}

We evaluate future necessity on $265$ usable compression boundary samples from $31$ tasks, with $112$ positive behavioral labels. Under strict task level nested evaluation, the probe reaches an AUROC of $0.844$, an AP of $0.811$, and a balanced accuracy of $0.760$. The prompt last representation is selected in all outer folds, while the selected layers span middle and later parts of the network. These results show that the current representation contains a readable signal associated with later rederivation, additional reasoning, and replanning. Full cohort and layer selection details are reported in Appendix~\ref{app:probe_protocol}.

Activation patching provides complementary evidence at the representation level. On the $47$ sample patching cohort, replacing deleted history residuals with their full history counterparts progressively restores the next token distribution. Mean KL Repair rises from $0.047$ at layer 4 to $0.459$ at layer 20, $0.724$ at layer 28, and $0.978$ at layer 31. The complete layerwise results, confidence intervals, and top one agreement values are reported in Appendix~\ref{app:patching}, Table~\ref{tab:patching_full}.

\subsection{State Externalization and Controlled Interventions}

Trajectory analysis provides a state based explanation for when historical reasoning remains useful. When task specific derived state is still available only in reasoning, the future behavioral risk rate is $67.3\%$, compared with $36.2\%$ when that state has already been externalized. The within task increase is $19.8$ percentage points, and the corresponding hidden score difference is $0.283$. Detailed group statistics and task level intervals are reported in Appendix~\ref{app:gap}.

Controlled interventions support the same interpretation. Keeping only the latest reasoning preserves the full history rule score of $0.872$, whereas deleting only the latest reasoning reduces the score to $0.760$, and deleting all reasoning reduces it to $0.675$. Once relevant code has been written to disk, clearing subsequent reasoning leaves the rule score unchanged at $0.759$. Environmental feedback shows the same pattern. When an observed \texttt{stderr} remains available, the agent proceeds directly to Edit. When the same feedback is hidden, the agent issues another Bash call and observes the same error again. The fixed state experiments are reported in Appendix~\ref{app:controlled_interventions}, and the environmental feedback intervention is reported in Appendix~\ref{app:environment_memory}.

Taken together, these results suggest that historical reasoning is most useful while it remains the only reliable carrier of task relevant derived state. As that state is transferred into persistent artifacts or environmental feedback, the original reasoning becomes increasingly replaceable.

\section{Conclusion}
\label{sec:conclusion}

We study when a long horizon agent can forget its own reasoning. Unlike static Chain of Thought compression, online compression changes future decisions and the resulting trajectory. \method{} uses frozen proxy entropy to remove low entropy reasoning while preserving actions and observations, reducing token consumption and improving aggregate reward on the full 260 task evaluation. Ablations reveal trajectory amplification, where local deletion and final system savings are not proportional. Representation probing and activation patching show that current states encode information associated with later reasoning reuse, while controlled analyses indicate that historical reasoning becomes more replaceable as task relevant derived state moves into reliable external carriers. These results more clearly characterize agent reasoning as dynamic working state whose value changes throughout interaction.

\subsection*{AI use statement}

Generative AI tools were used for language editing, drafting assistance, and literature discovery during preparation of this manuscript. All AI assisted text, claims, citations, and experimental descriptions were reviewed by the authors. Experimental results and interpretations were checked against the underlying records, and the authors take full responsibility for the final content of the paper.

\subsection*{Reproducibility statement}

The main paper specifies the online compression protocol, proxy entropy
definition, pruning rule, benchmark composition, and principal ablations.
The appendix provides implementation details, evaluation set composition,
extended system and ablation results, hidden state probing and activation
patching protocols, controlled state interventions, and theoretical analysis.

\clearpage
\setlength{\bibsep}{.5ex plus .8ex}
\bibliography{iclr2027_conference}
\bibliographystyle{unsrtnat}

\clearpage
\appendix
\numberwithin{equation}{section}
\setlength{\emergencystretch}{2em}

\section{The Use of Large Language Models}
\label{app:llm_use}

Generative AI tools were used for language editing, drafting assistance, literature discovery, LaTeX preparation, and organization of supplementary explanations. All claims, citations, mathematical statements, and experimental descriptions were reviewed by the authors.

\paragraph{Scope.}
The supplementary material expands the definitions, protocols, theoretical properties, and experimental results reported in the main paper.

\FloatBarrier
\section{Theoretical Analysis}
\label{app:theory}

\subsection{Notation and scope}
\label{app:theory_scope}

At one compression boundary, let $C$ denote the context preceding the completed reasoning,
$B=(B_1,\ldots,B_N)$ the ordered reasoning blocks, and $U$ the unchanged nonreasoning records of the completed step, including its action, tool record, and returned observation. Let $Y$ denote a downstream target under a fixed reference distribution $P$. We write $B_{<j}=(B_1,\ldots,B_{j-1})$, $B_{-j}$ for all reasoning blocks except $B_j$, and $B_D$ for the chronologically ordered tuple indexed by $D$.

We use $\mathsf H_P$ and $I_P$ for population Shannon entropy and mutual information, while $H(B_j)$ in Section~\ref{sec:entropy_pruning} denotes the realized proxy score. The following results connect these quantities through the information structure of retained and removed reasoning.

\subsection{Residual information of a reasoning block}
\label{app:proof_single}

\begin{applemma}[Residual information bound]
\label{lem:residual_information}
For every block $B_j$,
\begin{equation}
I_P(B_j;Y\mid C,U,B_{-j})
\leq
\mathsf H_P(B_j\mid C,B_{<j}).
\label{eq:single_information_bound}
\end{equation}
\end{applemma}

\begin{proof}
By the definition of conditional mutual information,
\begin{align}
I_P(B_j;Y\mid C,U,B_{-j})
&=
\mathsf H_P(B_j\mid C,U,B_{-j})
-
\mathsf H_P(B_j\mid C,U,B_{-j},Y)
\\
&\leq
\mathsf H_P(B_j\mid C,U,B_{-j})
\\
&\leq
\mathsf H_P(B_j\mid C,B_{<j}).
\end{align}
The first inequality follows from nonnegativity of conditional entropy. The second follows because conditioning on $U$ and all blocks in $B_{-j}$ includes the earlier blocks $B_{<j}$ and can only reduce conditional entropy.
\end{proof}

\paragraph{Interpretation.}
Under the reference distribution, a block with low true conditional entropy has a correspondingly limited residual information contribution about $Y$. This bound motivates conditional uncertainty as a compression signal and provides an information theoretic view of why low uncertainty reasoning can be more replaceable.

\subsection{Several removed blocks}
\label{app:proof_subset}

\begin{apptheorem}[Fixed subset information bound]
\label{thm:subset_information}
For any index set $D\subseteq\{1,\ldots,N\}$ fixed independently of the realized block values,
\begin{equation}
I_P(B_D;Y\mid C,U,B_{\bar D})
\leq
\sum_{j\in D}
\mathsf H_P(B_j\mid C,B_{<j}),
\label{eq:subset_information_bound}
\end{equation}
where $\bar D$ is the complement of $D$.
\end{apptheorem}

\begin{proof}
Write $D=\{j_1<\cdots<j_m\}$. Then
\begin{align}
I_P(B_D;Y\mid C,U,B_{\bar D})
&\leq
\mathsf H_P(B_D\mid C,U,B_{\bar D})
\\
&=
\sum_{r=1}^{m}
\mathsf H_P
\left(
B_{j_r}\mid
C,U,B_{\bar D},B_{j_1},\ldots,B_{j_{r-1}}
\right)
\\
&\leq
\sum_{r=1}^{m}
\mathsf H_P(B_{j_r}\mid C,B_{<j_r}).
\end{align}
The first inequality bounds mutual information by conditional entropy. The equality is the entropy chain rule. In each term, the conditioning set contains every earlier block, either because that block is retained in $B_{\bar D}$ or because it appears among the previously selected blocks.
\end{proof}

\paragraph{Adaptive ranking.}
The implemented deletion set is selected from observed proxy entropy scores and is therefore data dependent. Theorem~\ref{thm:subset_information} characterizes the fixed subset information structure, while ICLR extends this principle through online ranking of observed reasoning blocks.

\subsection{Online trajectory perturbation}
\label{app:trajectory_bound}

The main paper emphasizes that a local edit changes the state from which later actions are generated. The following bound makes this dependence explicit. Consider a common finite horizon $K$ and two policies interacting with the same environment from the same initial state. At any common observable history $h$, let $\pi_t(\cdot\mid h)$ and $\widetilde\pi_t(\cdot\mid h)$ denote the next action laws of the base and compressed systems.

\begin{apptheorem}[Finite horizon trajectory perturbation]
\label{thm:trajectory_tv}
Suppose that for every jointly reachable common history,
\begin{equation}
\operatorname{TV}
\left(
\pi_t(\cdot\mid h),
\widetilde\pi_t(\cdot\mid h)
\right)
\leq
\epsilon_t,
\qquad 0\leq \epsilon_t\leq 1.
\end{equation}
Then the observable trajectory laws satisfy
\begin{equation}
\operatorname{TV}
\left(
\mathbb P_\pi,
\mathbb P_{\widetilde\pi}
\right)
\leq
1-\prod_{t=1}^{K}(1-\epsilon_t)
\leq
\min\left\{1,\sum_{t=1}^{K}\epsilon_t\right\}.
\label{eq:trajectory_tv}
\end{equation}
For any terminal score $F\in[0,1]$,
\begin{equation}
\left|
\mathbb E_\pi F-\mathbb E_{\widetilde\pi}F
\right|
\leq
\min\left\{1,\sum_{t=1}^{K}\epsilon_t\right\}.
\label{eq:reward_tv}
\end{equation}
\end{apptheorem}

\begin{proof}
While the two histories agree, couple their next actions maximally. The coupled actions agree with probability at least $1-\epsilon_t$. Conditional on equal actions and equal observable histories, couple the next observation identically through the common environment transition. Induction gives a probability of complete trajectory agreement of at least $\prod_t(1-\epsilon_t)$. Total variation is bounded by the mismatch probability of any coupling, which gives the first inequality. The second follows from the union bound. Since $F\in[0,1]$, its values can differ by at most one on the mismatch event, which yields Equation~\eqref{eq:reward_tv}.
\end{proof}

\paragraph{Interpretation.}
This result formalizes how repeated local policy changes can accumulate across an interaction trajectory and provides a theoretical basis for trajectory amplification.

\subsection{A sufficient condition for forgetting externalized state}
\label{app:externalization_theory}

Let $Z_t$ denote task relevant state that is available from the retained external interaction history, and let $W_t$ denote reasoning content that would be removed.

\begin{appcorollary}[Externalized state sufficiency]
\label{cor:externalization}
Suppose that at every reachable history the full history and compressed policies share the same retained state $Z_t$, and that
\begin{equation}
\pi_t(a\mid Z_t,W_t)
=
q_t(a\mid Z_t)
=
\widetilde\pi_t(a\mid Z_t).
\label{eq:externalized_policy}
\end{equation}
Under the common environment assumptions of Theorem~\ref{thm:trajectory_tv}, the two systems induce identical observable trajectory and terminal score distributions.
\end{appcorollary}

\begin{proof}
Equation~\eqref{eq:externalized_policy} implies zero total variation between the two next action laws at every common history. Thus $\epsilon_t=0$ for all $t$ in Theorem~\ref{thm:trajectory_tv}, and both the trajectory and terminal score distances are zero.
\end{proof}

\paragraph{Interpretation.}
The corollary gives a sufficient condition under which externalized state can fully replace discarded reasoning. The controlled experiments in Section~\ref{sec:experiments} examine local cases in which useful derived state moves from reasoning into code, files, tool outputs, or environmental feedback.

\subsection{Exact token accounting and trajectory amplification}
\label{app:token_decomposition}

Let
\begin{equation}
T^{\mathrm b}
=
\sum_{k=1}^{K_b}
\left(
I_k^{\mathrm b}+O_k^{\mathrm b}
\right),
\qquad
T^{\mathrm c}
=
\sum_{k=1}^{K_c}
\left(
I_k^{\mathrm c}+O_k^{\mathrm c}
\right)
\end{equation}
be the processed token totals of the base and compressed trajectories, where $I_k$ and $O_k$ denote input and output tokens of request $k$.

\begin{appproposition}[Trajectory level token decomposition]
\label{prop:token_decomposition}
Let $K_*=\min(K_b,K_c)$. Then
\begin{align}
T^{\mathrm c}-T^{\mathrm b}
={}&
\sum_{k=1}^{K_*}
\left[
(I_k^{\mathrm c}-I_k^{\mathrm b})
+
(O_k^{\mathrm c}-O_k^{\mathrm b})
\right]
\nonumber\\
&+
\sum_{k=K_*+1}^{K_c}
(I_k^{\mathrm c}+O_k^{\mathrm c})
-
\sum_{k=K_*+1}^{K_b}
(I_k^{\mathrm b}+O_k^{\mathrm b}).
\label{eq:token_decomposition}
\end{align}
\end{appproposition}

\begin{proof}
Split each total into the first $K_*$ requests and its unmatched suffix, then subtract the two sums term by term.
\end{proof}

\paragraph{Interpretation.}
Direct deletion affects the serialized context of one request, while Equation~\eqref{eq:token_decomposition} also captures changes to later input, later output, and the number of requests. Whole trajectory savings therefore depend on the induced interaction path as well as the local deletion itself. This identity gives the accounting basis for the trajectory amplification analysis in Section~\ref{sec:experiments}.

\FloatBarrier
\section{Implementation and Online Compression Protocol}
\label{app:implementation}

\subsection{Runtime settings}
\label{app:runtime}

\begin{table}[htbp]
\centering
\small
\setlength{\tabcolsep}{5pt}
\renewcommand{\arraystretch}{1.10}
\caption{Recorded execution settings for the acting agent and frozen proxy scorer.}
\label{tab:runtime_settings}
\begin{tabular}{@{}l>{\raggedright\arraybackslash}p{0.58\linewidth}@{}}
\toprule
Component & Setting \\
\midrule
Acting model & DeepSeek-V4-Flash \\
Thinking / reasoning effort & Enabled / \texttt{high} \\
Requested temperature & $1$ \\
Requested maximum tokens & $384{,}000$ \\
Streaming & Enabled \\
Frozen scorer & \texttt{acm-browsecompplus-qwen3.5-9b-}\newline
\texttt{opd-iter3} \\
Scorer execution & \texttt{eval()} and \texttt{inference\_mode()} \\
Full context scorer limit & $42{,}000$ tokens \\
Deletion rule & Lowest $\lfloor0.8N\rfloor$ nonempty reasoning blocks \\
Replacement text & Literal \texttt{[SKIP]} \\
\bottomrule
\end{tabular}
\end{table}

\paragraph{Interpretation.}
The acting model generates actions and reasoning, while the frozen scorer is used only to compute the ranking signal. No scorer update is performed during evaluation. When the scorer input exceeds its $42$K limit, pruning is skipped for that step rather than truncating the acting agent's history.

\subsection{Reasoning extraction and entropy alignment}
\label{app:token_alignment}

Reasoning is taken from \texttt{thinking} when available. The remaining fallbacks are
\texttt{reasoning\_\allowbreak content}, \texttt{reasoning}, and a
\texttt{<think>...</think>} parser. Nonempty segments separated by
\texttt{\textbackslash n\textbackslash n} define candidate blocks. Only reasoning spans are editable. Actions, tool arguments, tool outputs, observations, and final answers remain unchanged.

For scorer token $x_t$, let $z_{t-1}$ denote the logits that predict that token. The implementation computes
\begin{align}
\ell_t(v)
&=
\operatorname{logsoftmax}(z_{t-1})_v,
\\
h_t
&=
-\frac{1}{\ln2}
\sum_{v\in\mathcal V}
\exp(\ell_t(v))\ell_t(v),
\\
H(B_j)
&=
|\mathcal I_j|^{-1}
\sum_{t\in\mathcal I_j}
h_t,
\label{eq:implementation_entropy}
\end{align}
where $\mathcal I_j$ is the set of scorer token positions assigned to block $B_j$.

\paragraph{Interpretation.}
Equation~\eqref{eq:implementation_entropy} is predictive entropy rather than observed token surprisal. The score is averaged within each block and used only for ranking candidate blocks.

\subsection{Online write back}
\label{app:online_protocol}

The first agent request is executed without compression. Before each subsequent request, the reasoning from the just completed interaction step is partitioned, scored, and rewritten in persistent history. All later requests consume this rewritten history rather than the original reasoning.

For $N$ nonempty reasoning blocks, ICLR replaces the lowest $\lfloor0.8N\rfloor$ blocks with \texttt{[SKIP]}. When $N=1$, no block is removed. The nominal fraction is defined over blocks, so the realized token reduction varies with block length.

\subsection{Intervention definitions}
\label{app:interventions}

\begin{table}[htbp]
\centering
\small
\setlength{\tabcolsep}{5pt}
\renewcommand{\arraystretch}{1.10}
\caption{Intervention families used in the ablation study.}
\label{tab:intervention_definitions}
\begin{tabular}{@{}l>{\raggedright\arraybackslash}p{0.57\linewidth}@{}}
\toprule
Variant & Operation \\
\midrule
Full history entropy & Score with the causal full context and replace the lowest $80\%$ of reasoning blocks. \\
Step only entropy & Score the just completed reasoning without earlier trajectory context. \\
Delete all thinking & Remove the complete targeted reasoning while preserving actions and observations. \\
Random $r$ & Remove a deterministic random subset of reasoning tokens for $r\in\{0.05,0.10,0.20,0.40\}$. \\
\bottomrule
\end{tabular}
\end{table}

\paragraph{Interpretation.}
The controls separate three questions. Delete all thinking tests wholesale removal, Random $r$ isolates deletion without entropy ranking, and Step only entropy measures the contribution of earlier trajectory context to scoring. Random removal operates at token level, whereas ICLR operates at block level, so the percentages describe different intervention granularities.

\FloatBarrier
\section{Evaluation Protocol}
\label{app:evaluation}

\subsection{Evaluation sets}
\label{app:datasets}

\begin{table}[htbp]
\centering
\small
\setlength{\tabcolsep}{4.5pt}
\renewcommand{\arraystretch}{1.10}
\caption{Evaluation sets used for the reported system and representation analyses.}
\label{tab:evaluation_sets}
\begin{tabular}{@{}l>{\raggedright\arraybackslash}p{0.58\linewidth}@{}}
\toprule
Set & Scope \\
\midrule
Full benchmark & $260$ tasks: Code $80$, Office $50$, Security $60$, Web $70$ \\
Pilot40 & Fixed $40$ task context management comparison \\
Ablation set & Fixed $80$ tasks: $20$ per domain \\
Representation cohort & $265$ usable boundary samples from $31$ tasks \\
Patching cohort & $47$ samples from $8$ tasks \\
Fixed state diagnostics & Controlled Office and error state continuations \\
\bottomrule
\end{tabular}
\end{table}

\paragraph{Interpretation.}
These rows correspond to different evaluation units. Whole task results measure end to end agent behavior, while representation and fixed state analyses operate on local states sampled from trajectories. Each analysis is reported with its corresponding denominator.

\subsection{Metrics and token accounting}
\label{app:token_accounting}

For a consistently defined usage metric $X$, we report the relative change
\begin{equation}
\Delta_X
=
100
\left(
\frac{X_{\rm method}}{X_{\rm base}}
-1
\right).
\label{eq:token_change}
\end{equation}
The processed token total in the main tables is input plus output tokens. Cache read usage is reported separately rather than added again to input.

Task performance is measured by the official WorkBuddyBench reward. Domain means are reported together with the aggregate task result. Wins, ties, and losses compare the method and base system on the same benchmark tasks.

\subsection{Task level uncertainty}
\label{app:statistics}

For the externalization analysis, reported confidence intervals use task level resampling so that multiple states from the same task remain grouped. If $\widehat\theta$ is a statistic computed from $M$ task clusters, a percentile interval has the form
\begin{equation}
\left[
Q_{0.025}(\widehat\theta^*),
Q_{0.975}(\widehat\theta^*)
\right],
\end{equation}
where each bootstrap replicate resamples complete tasks with replacement. The appendix reports the recorded task level intervals for the externalization analysis.

\FloatBarrier
\section{Extended System Results}
\label{app:system_results}

\subsection{Full benchmark domain results}
\label{app:full260}

\begin{table*}[htbp]
\centering
\small
\setlength{\tabcolsep}{4.0pt}
\renewcommand{\arraystretch}{1.10}
\caption{Full260 domain results. Rewards are shown on the original $[0,1]$ scale. Input and output columns report relative changes from the base system.}
\label{tab:full260_extended}
\begin{tabular*}{\textwidth}{@{\extracolsep{\fill}}lrrrrrrl@{}}
\toprule
Domain & $n$ & Base & ICLR & $\Delta R$ & $\Delta I$ (\%) & $\Delta O$ (\%) & W/T/L \\
\midrule
Code & 80 & 0.770000 & 0.706329 & $-0.063671$ & $+6.35$ & $+16.21$ & 16/40/24 \\
Office & 50 & 0.818000 & 0.776152 & $-0.041848$ & $-2.50$ & $-2.55$ & 16/8/26 \\
Security & 60 & 0.478000 & 0.706816 & $+0.228816$ & $-35.23$ & $-31.02$ & 39/9/12 \\
Web & 70 & 0.721000 & 0.698551 & $-0.022449$ & $-37.16$ & $-19.04$ & 31/19/20 \\
\midrule
Overall & 260 & 0.698654 & 0.717775 & $+0.019121$ & $-25.48$ & $-14.41$ & 102/76/82 \\
\bottomrule
\end{tabular*}
\end{table*}

\paragraph{Interpretation.}
The aggregate reward increases from $0.699$ to $0.718$, with the largest positive domain change appearing in Security. Code, Office, and Web show smaller domain scores than their base runs, and Code also uses more input. The complete 260 task aggregate combines these domain specific effects. The aggregate cache read reduction is $33.28\%$.

\subsection{Pilot40 precision}
\label{app:pilot40}

\begin{table}[htbp]
\centering
\small
\setlength{\tabcolsep}{4.5pt}
\renewcommand{\arraystretch}{1.10}
\caption{Recorded Pilot40 precision for ICLR.}
\label{tab:pilot40_precision}
\begin{tabular}{@{}rrrrrr@{}}
\toprule
Code & Office & Security & Web & Mean score & Total $I+O$ \\
\midrule
74.10 & 81.06 & 64.11 & 62.00 & 70.32 & 80,812,679 \\
\bottomrule
\end{tabular}
\end{table}

\paragraph{Interpretation.}
The Pilot40 table corresponds to the fixed forty task context management comparison used in Table~\ref{tab:pilot40}, while Full260 reports the complete benchmark evaluation.

\FloatBarrier
\section{Ablation Details}
\label{app:ablations}

\subsection{Reward summary}
\label{app:ablation_rewards}

\begin{table*}[htbp]
\centering
\small
\setlength{\tabcolsep}{5.0pt}
\renewcommand{\arraystretch}{1.10}
\caption{Ablation reward summary on the fixed 80 task set. Values correspond to the main ablation table and are shown on the original $[0,1]$ scale.}
\label{tab:ablation_appendix}
\begin{tabular*}{\textwidth}{@{\extracolsep{\fill}}lrrrrr@{}}
\toprule
Intervention & Code & Office & Security & Web & Average \\
\midrule
Base Agent & 0.6909 & 0.7894 & 0.4417 & 0.5850 & 0.6267 \\
ICLR & 0.6385 & 0.8114 & 0.7203 & 0.7800 & 0.7376 \\
Delete All Thinking & 0.6920 & 0.8204 & 0.6741 & 0.6850 & 0.7178 \\
Random $5\%$ & 0.5981 & 0.8112 & 0.6417 & 0.6683 & 0.6798 \\
Random $10\%$ & 0.5508 & 0.8265 & 0.7191 & 0.6061 & 0.6756 \\
Random $20\%$ & 0.6723 & 0.7692 & 0.6482 & 0.6150 & 0.6762 \\
Random $40\%$ & 0.6135 & 0.7617 & 0.6137 & 0.7162 & 0.6763 \\
Step Only Entropy & 0.6048 & 0.8058 & 0.7310 & 0.7038 & 0.7114 \\
\bottomrule
\end{tabular*}
\end{table*}

\paragraph{Interpretation.}
ICLR has the highest reported average reward in this ablation summary. Delete All Thinking also performs above the displayed base average, showing that the history contains substantial removable redundancy. Step Only Entropy remains competitive but below full history entropy, which is consistent with the use of earlier trajectory context in the scoring rule. The random controls show that deletion itself can change task performance, with outcomes depending on more than the configured random removal percentage.

\subsection{Token reduction and trajectory amplification}
\label{app:ablation_tokens}

\begin{table}[htbp]
\centering
\small
\setlength{\tabcolsep}{4.5pt}
\renewcommand{\arraystretch}{1.10}
\caption{Input and output reductions used in the trajectory amplification analysis.}
\label{tab:ablation_usage_changes}
\begin{tabular}{@{}lrr@{}}
\toprule
Intervention & Input reduction (\%) & Output reduction (\%) \\
\midrule
ICLR & 17.03 & 12.05 \\
Delete All Thinking & 63.59 & 13.55 \\
Random $5\%$ & 53.68 & 38.46 \\
Random $10\%$ & 60.53 & 46.74 \\
Random $20\%$ & 44.01 & 26.35 \\
Random $40\%$ & 51.21 & 34.22 \\
Step Only Entropy & 67.88 & 38.87 \\
\bottomrule
\end{tabular}
\end{table}

\paragraph{Interpretation.}
Input reduction varies nonmonotonically with the configured local deletion ratio. Random $10\%$ reduces input by $60.53\%$, while Random $20\%$ reduces input by $44.01\%$. This is the empirical pattern referred to as trajectory amplification. Proposition~\ref{prop:token_decomposition} explains this behavior through changes to subsequent requests as well as the current serialized history.

\paragraph{Relation to Figure~\ref{fig:trajectory}.}
The main figure visualizes the mismatch between local intervention strength and whole trajectory consequences, showing that final token and reward outcomes depend on the trajectory induced after the intervention.

\FloatBarrier
\section{Hidden State Probing}
\label{app:probing}

\subsection{Cohort and target}
\label{app:probe_protocol}

A frozen Qwen3.5 9B model reads the context visible at each compression boundary. For sample $i$, layer $l$, and pooling rule $p$,
\begin{equation}
z_i^{(l,p)}
=
\operatorname{Pool}_p
\left(
\mathcal Z_i^{(l)}
\right),
\qquad
\widehat p_i
=
\sigma
\left(
w^\top z_i^{(l,p)}+b
\right).
\label{eq:probe_definition}
\end{equation}
The behavioral target records whether the later trajectory contains rederivation, additional reasoning, or explicit replanning and serves as an operational measure of future reasoning reuse.

\begin{table}[htbp]
\centering
\small
\setlength{\tabcolsep}{4.5pt}
\renewcommand{\arraystretch}{1.10}
\caption{Representation extraction cohort used for future necessity probing.}
\label{tab:probe_cohort}
\begin{tabular}{@{}lr@{}}
\toprule
Quantity & Value \\
\midrule
Usable boundary samples & 265 \\
Distinct tasks & 31 \\
Positive behavioral labels & 112 \\
Transformer layers & 32 \\
Pooling candidates & 5 \\
Outer task folds & 5 \\
\bottomrule
\end{tabular}
\end{table}

\paragraph{Interpretation.}
The unit of analysis is a compression boundary rather than a complete benchmark task. Multiple boundary samples can come from the same task, which is why model selection and evaluation are separated at the task level.

\subsection{Nested task level evaluation}
\label{app:probe_nested_protocol}

Layer selection, pooling choice, feature standardization, and PCA are chosen using training tasks only. The outer folds contain held out tasks. This design prevents samples from the same task from appearing on both sides of the outer evaluation split.

\begin{table}[htbp]
\centering
\small
\setlength{\tabcolsep}{4.5pt}
\renewcommand{\arraystretch}{1.10}
\caption{Strict nested task level probe results.}
\label{tab:probe_nested}
\begin{tabular}{@{}rrrl@{}}
\toprule
AUROC & AP & Balanced accuracy & Selected pooling \\
\midrule
0.8444 & 0.8108 & 0.7597 & \texttt{prompt\_last} \\
\bottomrule
\end{tabular}
\end{table}

\paragraph{Interpretation.}
The probe results show that the current hidden state contains a readable cross task signal associated with later reasoning reuse. The linear probe serves as an analysis instrument for characterizing that signal.

\FloatBarrier
\section{Activation Patching}
\label{app:patching}

\begin{figure}[t]
\centering
\includegraphics[width=0.78\columnwidth]{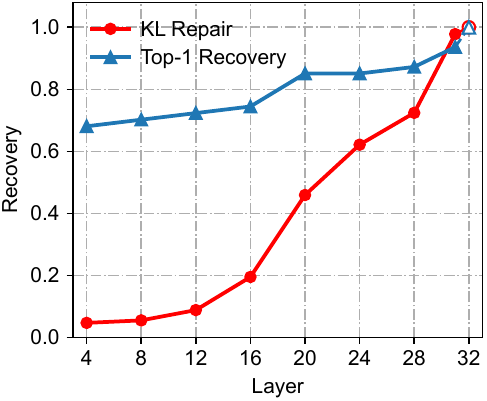}
\caption{\textbf{Activation patching across depth.}
Full history residuals progressively restore the perturbed next token distribution when they replace the corresponding deleted history residual at the prompt boundary.}
\label{fig:activation_patching}
\end{figure}

\paragraph{Figure interpretation.}
The curve summarizes a local intervention on the frozen representation sensor. Larger KL Repair means that the patched representation produces a next token distribution closer to the full history condition. Recovery increases toward later layers, showing that full history information is progressively expressed in representations closer to the output readout.

\subsection{Intervention and measurement}
\label{app:patch_protocol}

The patching analysis uses $47$ samples from $8$ tasks on the frozen Qwen3.5 9B sensor. For sample $i$, let $P_i^{\rm full}$ be the next token distribution under full history, $P_i^{\rm del}$ the distribution after historical reasoning is removed, and $P_{i,l}^{\rm patch}$ the distribution obtained after replacing the deleted history prompt boundary residual at layer $l$ with the corresponding full history residual.

Define
\begin{equation}
d_i
=
D_{\rm KL}
\left(
P_i^{\rm full}
\Vert
P_i^{\rm del}
\right).
\end{equation}
For $d_i>0$, KL Repair is
\begin{equation}
\operatorname{Repair}_i(l)
=
1-
\frac{
D_{\rm KL}
\left(
P_i^{\rm full}
\Vert
P_{i,l}^{\rm patch}
\right)
}{
d_i
}.
\label{eq:sample_repair}
\end{equation}
A value near one indicates strong distributional restoration relative to the deleted history condition.

\subsection{Layerwise results}
\label{app:patch_values}

\begin{table}[htbp]
\centering
\small
\setlength{\tabcolsep}{4.5pt}
\renewcommand{\arraystretch}{1.10}
\caption{Activation patching results on the $47$ sample cohort.}
\label{tab:patching_full}
\begin{tabular}{@{}rrrr@{}}
\toprule
Layer & Mean KL Repair & $95\%$ CI & Top 1 agreement \\
\midrule
4 & 0.047 & [0.027, 0.076] & 0.681 \\
16 & 0.195 & [0.170, 0.235] & 0.745 \\
20 & 0.459 & [0.392, 0.537] & 0.851 \\
24 & 0.622 & [0.545, 0.700] & 0.851 \\
28 & 0.724 & [0.667, 0.788] & 0.872 \\
31 & 0.978 & [0.964, 0.986] & 0.936 \\
32 & 1.000 & [1.000, 1.000] & 1.000 \\
\bottomrule
\end{tabular}
\end{table}

\paragraph{Interpretation.}
Repair rises substantially from early to late layers, matching the main text summary at layers 4, 20, 28, and 31. Top 1 agreement increases in the same direction. Layer 32 is an implementation positive control because patching the final residual immediately before a deterministic readout restores the full history readout input.

\begin{applemma}[Final layer patching control]
\label{lem:last_layer_patch}
Suppose that the final residual $h_L$ is followed by the same deterministic readout map $g$ in the full and deleted history conditions. If the deleted run is patched with $h_L^{\rm full}$, then
\begin{equation}
P_{i,L}^{\rm patch}
=
P_i^{\rm full}.
\end{equation}
Whenever $d_i>0$, this implies $\operatorname{Repair}_i(L)=1$.
\end{applemma}

\begin{proof}
After patching, both conditions pass the same residual $h_L^{\rm full}$ through the same deterministic map $g$. Their resulting next token distributions are therefore identical. Substituting the equality into Equation~\eqref{eq:sample_repair} gives zero numerator and repair equal to one.
\end{proof}

\FloatBarrier
\section{Externalization and Fixed State Diagnostics}
\label{app:externalization}

\subsection{Internalization and externalization gap}
\label{app:gap}

Let $\mathcal U_i$ denote the unexternalized task specific derived state at compression boundary $i$. We define
\begin{equation}
G_i
=
\mathbb{I}
\left(
\mathcal U_i\neq\varnothing
\right).
\label{eq:semantic_gap}
\end{equation}
Derived state includes intermediate plans, relations, constraints, syntheses, or task specific conclusions that remain internal to reasoning before materialization in a persistent external carrier.

\begin{table}[htbp]
\centering
\small
\setlength{\tabcolsep}{4.5pt}
\renewcommand{\arraystretch}{1.10}
\caption{Grouping by the internalization and externalization gap.}
\label{tab:externalization_gap}
\begin{tabular}{@{}lrrr@{}}
\toprule
Group & $n$ & Future risk & Mean hidden score \\
\midrule
No gap & 213 & $36.15\%$ & 0.4369 \\
Gap present & 52 & $67.31\%$ & 0.6662 \\
\bottomrule
\end{tabular}
\end{table}

\paragraph{Interpretation.}
Future behavioral risk is higher when the relevant derived state remains available only in reasoning. The hidden score shows the same ordering. These pooled values are descriptive and motivate the within task comparison below.

\begin{table}[htbp]
\centering
\small
\setlength{\tabcolsep}{4.5pt}
\renewcommand{\arraystretch}{1.10}
\caption{Within task contrasts for gap present minus no gap states.}
\label{tab:gap_paired_effects}
\begin{tabular}{@{}lrr@{}}
\toprule
Quantity & Mean difference & Task bootstrap $95\%$ CI \\
\midrule
Future behavioral risk & $+19.75$ pp & [1.00, 37.05] \\
Hidden score & $+0.2831$ & [0.1703, 0.3977] \\
\bottomrule
\end{tabular}
\end{table}

\paragraph{Interpretation.}
The within task contrasts preserve the same direction after comparing states from the same tasks, strengthening the state based interpretation in the main text.

\subsection{Controlled state interventions}
\label{app:controlled_interventions}

\begin{table*}[htbp]
\centering
\small
\setlength{\tabcolsep}{7pt}
\renewcommand{\arraystretch}{1.10}
\caption{Controlled state interventions. The three groups answer different questions about where useful task state is stored.}
\label{tab:controlled_state_interventions}
\begin{tabular*}{\textwidth}{@{\extracolsep{\fill}}lrrr@{}}
\toprule
Intervention & Rule score & Reasoning & New requests \\
\midrule
\rowcolor{groupgray}
\multicolumn{4}{c}{\textbf{Fixed intermediate Office state}} \\
\midrule
Keep & 0.8717 & 1,817 & 5.00 \\
Delete all & 0.8146 & 7,491 & 5.33 \\
Random $5\%$ & 0.8756 & 516 & 4.00 \\
Random $10\%$ & 0.8756 & 562 & 7.33 \\
\midrule
\rowcolor{groupgray}
\multicolumn{4}{c}{\textbf{Position intervention}} \\
\midrule
Keep all & 0.8717 & -- & -- \\
Delete all & 0.6749 & -- & -- \\
Keep latest only & 0.8717 & -- & -- \\
Delete latest only & 0.7598 & -- & -- \\
\midrule
\rowcolor{groupgray}
\multicolumn{4}{c}{\textbf{After code is written to disk}} \\
\midrule
Keep thinking & 0.7587 & 1,289 & -- \\
Clear thinking & 0.7587 & 3,222 & -- \\
\bottomrule
\end{tabular*}
\end{table*}

\paragraph{Interpretation.}
The fixed intermediate state shows that complete deletion can trigger substantially more reasoning while lowering the rule score. The position intervention shows that keeping only the latest reasoning matches the keep all score, whereas deleting the latest reasoning causes a larger drop. After relevant code has been written to disk, clearing subsequent reasoning leaves the rule score unchanged at $0.7587$. Together, these local interventions support the claim that the importance of reasoning depends on whether task relevant state has already been externalized.

\subsection{Environmental feedback intervention}
\label{app:environment_memory}

\begin{table}[htbp]
\centering
\small
\setlength{\tabcolsep}{4.5pt}
\renewcommand{\arraystretch}{1.10}
\caption{Next action after retaining or hiding the already observed \texttt{stderr}.}
\label{tab:stderr_intervention}
\begin{tabular}{@{}lll>{\raggedright\arraybackslash}p{0.28\linewidth}@{}}
\toprule
Feedback & Next action & Count & Subsequent observation \\
\midrule
\texttt{stderr} visible & Edit & 9/9 & No repeated error observation \\
\texttt{stderr} hidden & Bash & 9/9 & The same \texttt{ValueError} appears again \\
\bottomrule
\end{tabular}
\end{table}

\paragraph{Interpretation.}
When the error remains visible, the agent can act directly on the observed state. When the same feedback is hidden, the agent first performs another Bash call and recovers the same error. This provides a concrete example in which environmental feedback functions as an external information carrier.

\FloatBarrier
\section{Motivation Figure Details}
\label{app:motivation}

Figure~\ref{fig:motivation} is used to motivate the state dependent value of reasoning. Panels (a) and (b) compare task score and token usage across different reasoning settings and difficulty levels. Panel (c) summarizes the empirical reward and input reduction tradeoff under online interventions.

\paragraph{Interpretation.}
The figure highlights two motivating patterns. First, the relation between reasoning intensity and task score varies across difficulty levels. Second, the relation between local deletion strength and end to end savings varies across interventions. Appendix~\ref{app:ablations} provides the corresponding intervention results.

\FloatBarrier
\section{Main Text to Appendix Map}
\label{app:evidence_map}

\begin{table*}[htbp]
\centering
\small
\setlength{\tabcolsep}{5pt}
\renewcommand{\arraystretch}{1.10}
\caption{Correspondence between major main text claims and supplementary evidence.}
\label{tab:evidence_map}
\begin{tabular*}{\textwidth}{@{\extracolsep{\fill}}p{0.30\textwidth}p{0.27\textwidth}p{0.32\textwidth}@{}}
\toprule
Main text claim & Supplementary location & Supporting material \\
\midrule
Online compression changes future interaction & Appendix~\ref{app:trajectory_bound} & Finite horizon perturbation theorem \\
Low entropy ranking is an information motivated heuristic & Appendices~\ref{app:proof_single} and~\ref{app:proof_subset} & Conditional information bounds and qualification \\
Local deletion and total computation differ & Appendix~\ref{app:token_decomposition} & Exact trajectory token decomposition \\
Full benchmark result & Appendix~\ref{app:full260} & Domain rewards, token changes, and W/T/L \\
Trajectory amplification & Appendix~\ref{app:ablation_tokens} & Intervention specific input and output reductions \\
Future necessity is decodable & Appendix~\ref{app:probing} & Cohort, nested protocol, and probe metrics \\
Historical reasoning affects representations & Appendix~\ref{app:patching} & Activation patching curve and layerwise results \\
Externalization changes reasoning usefulness & Appendix~\ref{app:externalization} & Gap analysis and controlled state interventions \\
\bottomrule
\end{tabular*}
\end{table*}

\paragraph{Interpretation.}
The table organizes the paper's supporting evidence into system level benchmark results, local mechanistic diagnostics, and mathematical analysis.

\FloatBarrier
\section{Broader Impact and Future Directions}
\label{app:broader_impact}

Reducing repeated processing of agent history lowers active context usage and can improve inference efficiency. ICLR preserves actions, tool calls, and observations while selectively compressing reasoning, keeping externally recorded task state available throughout interaction.

Several aspects of the current study motivate future work. Proxy entropy is computed by a frozen model distinct from the acting agent, motivating tighter alignment between scorer and policy representations. The future necessity target is behavioral, suggesting richer labels for different forms of reasoning reuse. Activation patching characterizes the frozen representation sensor, while future studies can extend the same analysis to additional acting models and longer closed loop interventions. Externalization analysis also motivates explicit modeling of when code, files, tool outputs, and observations become sufficient carriers of task relevant derived state.

Raw trajectory access should respect the confidentiality of prompts, credentials, tool outputs, and generated artifacts.

\FloatBarrier
\section{Task Manifest and Reproducibility Checklist}
\label{app:reproducibility}

\subsection{Task set summary}
\label{app:task_manifest}

\begin{table}[htbp]
\centering
\small
\setlength{\tabcolsep}{5pt}
\renewcommand{\arraystretch}{1.10}
\caption{Task set composition used in the main experiments.}
\label{tab:task_manifest_summary}
\begin{tabular}{@{}lrrrrr@{}}
\toprule
Set & Code & Office & Security & Web & Total \\
\midrule
Full benchmark & 80 & 50 & 60 & 70 & 260 \\
Ablation set & 20 & 20 & 20 & 20 & 80 \\
Pilot40 & \multicolumn{4}{c}{Fixed forty task comparison} & 40 \\
\bottomrule
\end{tabular}
\end{table}

\paragraph{Interpretation.}
The Full260, fixed 80 task ablation set, and Pilot40 comparison provide complementary views of complete benchmark performance, controlled ablation behavior, and context management comparison.

\subsection{Available records}
\label{app:missing_material}

\begin{table}[htbp]
\centering
\small
\setlength{\tabcolsep}{5pt}
\renewcommand{\arraystretch}{1.10}
\caption{Reproducibility checklist for the experiments reported in this paper.}
\label{tab:reproducibility_status}
\begin{tabular}{@{}>{\raggedright\arraybackslash}p{0.28\linewidth}>{\raggedright\arraybackslash}p{0.57\linewidth}@{}}
\toprule
Item & Available material \\
\midrule
Method definition & Block segmentation, proxy entropy, context conditioned scorer, deletion rule, placeholder, and online write back \\
Runtime & Acting model, frozen scorer, scorer context limit, and principal client settings \\
Evaluation & Full260, Pilot40, and fixed 80 task ablation composition and aggregate results \\
Probing & Cohort size, target definition, nested task level protocol, pooling choice, and aggregate metrics \\
Patching & Cohort size, intervention definition, KL Repair, confidence intervals, and layerwise results \\
State interventions & Gap statistics, fixed state rule scores, and environmental feedback intervention \\
Theory & Information bounds, trajectory perturbation bound, externalization sufficiency condition, and token decomposition \\
\bottomrule
\end{tabular}
\end{table}

\paragraph{Interpretation.}
The checklist summarizes the implementation, evaluation, representation analysis, intervention results, and theoretical material documented in the paper and appendix.

\end{document}